\documentclass[journal]{IEEEtran}
\usepackage[utf8]{inputenc}
\usepackage{graphicx}
\usepackage{multirow}
\usepackage{algorithm}
\usepackage{algpseudocode}
\usepackage{amsmath}
\usepackage{multicol}
\usepackage[export]{adjustbox}
\usepackage{float}
\usepackage{amssymb}
\usepackage{lipsum}
\usepackage{booktabs,siunitx}
\usepackage{cite}
\usepackage{comment}
\usepackage{color}
\usepackage{cases}
\usepackage{seqsplit}
\usepackage{mwe}
\usepackage{multirow}
\usepackage{epstopdf}
\usepackage[caption=false]{subfig}
\usepackage{caption}
\usepackage{adjustbox}
\usepackage{amsthm}
\usepackage{amssymb}
\usepackage{mathtools}
\usepackage{soul}
\DeclarePairedDelimiter\floor{\lfloor}{\rfloor}
\makeatletter
\newcommand*{\rom}[1]{\expandafter\@slowromancap\romannumeral #1@}
\makeatother
\definecolor{dukeblue}{rgb}{0.0, 0.0, 0.61}

\newtheorem{theorem}{Theorem}[section]

\newtheorem{proposition}[theorem]{Proposition}

\begin{document}
	\title{Efficient algorithms for electric vehicles' min-max routing problem}
	
	\author{Seyed Sajjad Fazeli, Saravanan Venkatachalam, Jonathon M. Smereka
		\thanks{
		S.S. Fazeli and S. Venkatachalam are members of the Department of Industrial and Systems Engineering at Wayne State University, Detroit, MI. J. M. Smereka is a researcher within the Ground Vehicle Robotics (GVR) team at the U.S. Army CCDC Ground Vehicle Systems Center (GVSC) in Warren, MI, 
		(e-mail: sajjad.fazeli@wayne.edu; saravanan.v@wayne.edu; jonathon.m.smereka.civ@mail.mil),
		Corresponding author: S.Venkatachalam. Distribution A. Approved for public release; distribution is unlimited. OPSEC \# 4492.}}
	\vspace{-15mm}
	\maketitle
	\begin{abstract}
		An increase in greenhouse gases emission from the transportation sector has led  companies and the government to elevate and support the production of electric vehicles (EV). With recent developments in urbanization and e-commerce, transportation companies are replacing their conventional fleet with EVs to strengthen the efforts for sustainable and environment-friendly operations. However, deploying a fleet of  EVs asks for efficient routing and recharging strategies to alleviate their limited range and mitigate the battery degradation rate.
		In this work, a fleet of electric vehicles is considered for transportation and logistic capabilities with limited battery capacity and scarce charging station availability. We introduce a min-max electric vehicle routing problem (MEVRP) where the maximum distance traveled by any EV is minimized while considering charging stations for recharging. We propose an efficient branch and cut framework and a  three-phase hybrid heuristic algorithm that can efficiently solve a variety of instances. Extensive computational results and sensitivity analyses are performed to corroborate the efficiency of the proposed approach, both quantitatively and qualitatively.
	\end{abstract}
	\IEEEpeerreviewmaketitle
	\vspace{-4mm}
	\section{Introduction}
	Global warming has been primarily linked to human activities which release greenhouse gases \cite{nazaripouya2019electric}. Among those activities, the transportation sector causes the largest share (about 28\%) of greenhouse gas emissions, which mainly originate from fossil fuel burner vehicles \cite{solaymani2019co2}. In an effort to offset carbon emissions from fossil fuel burning vehicles, priority to transform the transportation systems by driving new technological innovations in vehicles \cite{wu2019role} is causing electric vehicles (EVs) to become rapidly important for many automotive companies. Many countries have offered incentives to accelerate the adoption of EVs to increase the EV share in future vehicle fleets \cite{yi2018energy}.
	\par 
	
	It is important for EVs to choose energy-efficient routes and find the best locations for recharging during their itineraries. Transportation network companies (e.g. Lyft and Uber), and logistic companies have (e.g., FedEx and UPS) already started to operate a fleet of EVs in their business for last mile deliveries \cite{zhang2019joint}. Adopting EVs also brings new challenges. One of the main operational challenges for EVs in transport applications is their limited range and the availability of charging stations (CS) \cite{sundar2016exact,schneider2014electric,hiermann2016electric}. It is estimated that half of the US population lives in areas with fewer than 90 charging infrastructures per million people \cite{slowik2018continued,fazeli2020two}. To successfully employ EVs, we need strategies that can alleviate the range and recharging limitations. 
	
	The electric vehicle routing problem with limited range and number of CSs presented in this work, falls into the category of green vehicle routing problem (GVRP). The GVRP embraces a broad and extensive class of problems considering environmental issues as well as finding the best possible routes for vehicles. The GVRP research can be broadly divided into two categories: 1) minimize the fuel consumption while considering loading weights \cite{kocc2016green}; and 2) replace the conventional vehicles with alternative fuel vehicles (AFV) \cite{schneider2014electric,sundar2013algorithms,desaulniers2016exact,vincent2017simulated,juan2014routing}. This research focuses on AFVs, hence we briefly review the related literature on routing strategies for AFVs. An initial work was done by \cite{erdougan2012green}, where the authors developed a mixed integer programming (MIP) formulation and a genetic algorithm (GA) to overcome the range limitation of AFVs and shortage of refueling locations. Authors in \cite{schneider2014electric} introduced the electric vehicle routing problem (EVRP) with time windows and charging stations with the limited freight capacity for the vehicles. They developed a hybrid meta-heuristic by integrating variable neighborhood and Tabu search. For single unmanned aerial vehicles' routing problem, authors in \cite{sundar2013algorithms}  proposed a novel approach based on an approximation algorithm combined with a heuristic method. Later, the authors extended their work to multiple vehicles by applying the Voronoi algorithm as the construction phase, and 2-opt, 3-opt variable neighborhood searches in the improvement phase \cite{levy2014heuristics}.
	\par
	
	In this work, the MEVRP is defined as follows: given a set of EVs which are initially stationed at a depot, a set of targets and a set of CSs, the goal is to visit each target exactly once by any EV and return to the depot while no EV runs out of charge while they travel their respective routes. It is assumed that all EVs will be fully charged at any visited CS, and the fuel consumption rate is linearly proportioned with traveled distance. The objective of this problem is to minimize the maximum distance traveled by any EV instead of the total distance, which is conventional in VRP. The MEVRP is fundamentally different from the EVRP. An optimal solution in MEVRP assigns routes to all EVs such that none of the EVs has a longer route. This results in a balanced distribution of loads and fair and equitable utilization of the EVs, which can decrease the rates of battery degradation in EVs \cite{lunz2012influence}. Refer to Fig. \ref{intro} for an illustration for MEVRP and EVRP routes for EVs. In Fig. \ref{intro}, EV1 visits most of the targets and travels a lot of distance compared to EV2 with min-sum. However, using  MEVRP, with a nominal  increase in the overall distance, the two EVs travel almost the same amount of distance. In addition, the number of recharging of EVs in more evenly distributed in the min-max comparing to min-sum.
	\begin{figure}
		\centering
		\includegraphics[scale=0.25]{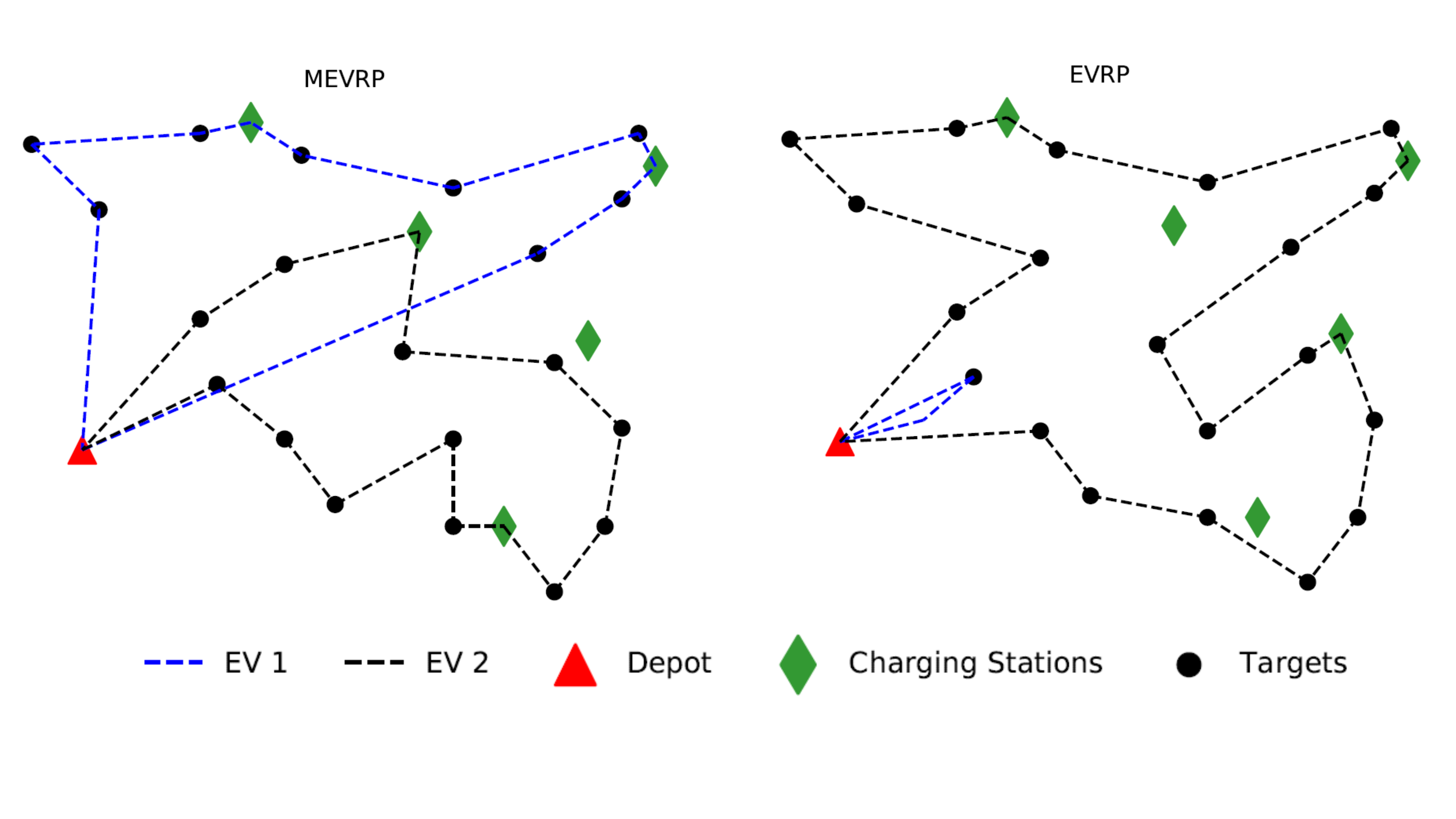}
		\vspace{-2mm}
		\caption {A feasible tour for two EVs visiting all the targets while visiting some charging stations for recharging: MEVRP (left) vs EVRP (right).}
		\label{intro}
	\end{figure}
	
	The MEVRP is also of interest when the time to visit every target from the base depot is more important than the total traveled distance in applications like surveillance, emergency and disaster management \cite{campbell2008routing}, intelligence, and reconnaissance \cite{torabbeigi2020drone,zaloga2011unmanned,manyam2016path}, and multi-robot path planning problems. In energy-efficient multi-robot path planning, the goal is to obtain optimal paths for each robot while avoiding obstacles in the presence of recharging points \cite{kapoutsis2017darp}. Additionally, the min-max provides a fair and equitable utilization for the resources, and the maximum wait time for any target will be less compared to a solution from min-sum. This is particularly vital if the EVs are used to transport people, and the targets are considered as stops in the EVs' routes. There are quite a few studies in the min-max VRP which differ by solution methodologies. These problems are often solved by heuristic methods with multiple phases for constructing the initial solution(s), and subsequently, improving them. Methodologies differ based on the construction of initial solution and the number of base depots. The author in \cite{yakici2017heuristic} considered a single depot min-max vehicle routing problem (SDVRP) where they used an ant colony system as well as a random approach to assign targets to the vehicles. A 3-opt method is used to improve solutions. The work in \cite{sze2017cumulative} considered a capacitated SDVRP where they generated initial solutions by a parallel greedy insertion method and improved them by an adaptive variable neighborhood search (VNS). In multi-depot min-max VRP, usually, the authors use partitioning techniques to transfer the min-max MDVRP to a set of SDVRPs, and solve each SDVRP separately \cite{narasimha2013ant,wang2016min,carlsson2009solving}.\par
	To the best of our knowledge, this study is the first attempt to formulate and solve the min-max version of MEVRP where range limitation is defined for each EV with a set of charging stations. We propose an efficient MIP formulation to solve small-scale instances. For large-scale instances, we develop a hybrid heuristic algorithm (HHA) where we obtain initial solution using an integer programming model and a heuristic, and subsequently, Variable Neighborhood Search (VNS) and genetic algorithm are used to improve the solutions along with novel feasibility methods. Extensive computational experiments evaluate all the proposed approaches.\par
	
	The contributions of this study include the following: 1) an efficient MIP formulation for the MEVRP; 2) an HHA with an embedded feasibility method for large-scale instances and extensive computational experiments to quantify the efficacy of the proposed approach; 3) computational experiments for MIP formulation using branch and cut algorithm; 4) a sensitivity analysis to investigate the aspects of solutions from EVRP and MEVRP.\par
	
	The remainder of this paper is organized as follows: Section \ref{formula} provides a mathematical formulation of the problem along with a subsequent reformulation. Section \ref{Method} introduces the solution methodologies where we present exact and heuristic methods to solve small and large-scale instances, respectively. Section \ref{ER} presents extensive computational experiments and sensitivity analysis. Finally, Section \ref{Con} provides concluding remarks.
	\vspace{-6mm}
	\section{Model Formulation} \label{formula}
	\subsection{Problem Definition}
	We define $T$ as a set of targets, and $\bar{D}$  as a set of CSs. Define $D =\bar{D} \cup d_0$, be a set of CSs, including a depot $d_0$ where $m$ EVs are initially stationed, and each EV is charged to its battery capacity. The MEVRP is defined on a directed graph with a set of vertices $V$ and a set of edges $E$ as $G = (V,E)$ where $V=T\cap D$. We assume that the graph $G$ does not contain any self-loop. Each edge $(i,j)\in E$ is associated with a non-negative cost $c_{ij}$ between vertices $i$ and $j$. It is assumed to be directly proportional to the energy consumption $c_{i,j} = K\cdot f_{ij}^{m}$, where $f_{ij}^{m}$ is the amount of energy consumption by traveling from $i$ to $j$ and $K$ is a constant denoting the energy consumption rate of EV $m$. It is also assumed that both the distances and the charge costs satisfy the triangle inequality, e.g., $\forall i,j,k \in V,$ $f_{ij}^{m}+f_{jk}^{m}\geq f_{ik}^{m}$. Also, let $F_{m}$ denote the maximum charging capacity of any EV $m$. The objective of the model is to find a route for each EV starting and ending at the base depot such that:  each target is visited at least once by an EV; no EV runs out of charge during the trip; and maximum distance traveled by an EV is minimized. The objective functions of MEVRP is represented as follows: \text{Min \big($\underset{\forall m \in M}{\text{Max}}$ } $\sum_{(i,j) \in E}c_{ij}x_{ij}^{m}\big)$, and represented as a linear function in the following mathematical model.
	\vspace{-5mm}
	\subsection{Notation}
	\begin{itemize}
		\item Sets
		\begin{itemize}
			\item $T$: Set of targets, indexed as $t \in T$.
			\item $\bar{D}$: Set of charging stations, indexed as $d \in \bar{D}$.
			\item $D$: Set of charging stations and base depot, $D= \bar{D} \cup \{d_0 \}$.
			\item $V$: Set of all vertices in the graph, including all targets, CSs and base depot, $V = T \cup D$.
			\item $E$: Set of all edges connecting any two vertices without any self-loop,  $(i,j) \in E$ and $i,j \in V$. 
			\item $S$ : Subset of targets and a depot in $V$, $S \subset V$, $\sigma^{+}(S)=\{(i,j) \in E : i \in S, j \not\in S\}$. 
			\item $M$: Set of EVs which are initially stationed at base depot $d_0$, indexed by $m \in M$.
		\end{itemize} 
		\item Model parameters
		\begin{itemize}
			\item $c_{ij}$: Cost of traversing an edge $(i,j) \in E$.
			\item $f_{ij}^{m}$: Amount of energy consumption of EV $m$ by traveling from node $i$ to $j$ with $i,j \in V$.
			\item $F_m$: Battery capacity of EV $m$.
			\item $q$: A large constant, set as number of targets.
		\end{itemize}
		\item  Decision variables
		\begin{itemize}
			\item $x_{ij}^{m}$: 1 if the edge $(i, j)$ is traversed by an EV $m$, and 0 otherwise;
			\item $z_{ij}^{m}$: Flow variable associated with each edge $(i, j) \in E$ indicating the amount of distance traveled by any EV.
			\item $y_{d}^{m}$: 1 if the CS $d \in D$ is visited by any EV $m$, and 0 otherwise.
			\item $w$: Maximum traveled distance by any EV.
		\end{itemize}
	\end{itemize}
	\vspace{-4mm}
	\subsection{MEVRP Model}\label{model}
	\vspace{-3mm}
	\begin{alignat}{3}
	& \text{Min} \hspace{1mm} w  && \label{obj} \\ 
	& \text{s.t. } && \nonumber \\
	& w \geq \sum_{(i,j) \in E}c_{ij}x_{ij}^{m}  \hspace{0.5cm} \forall m \in M,   && \label{Cons1} \\  
	& \sum_{i \in V}{x_{di}^{m}}= \sum_{i \in V} {x_{id}^{m}} \hspace{1.6cm} \forall   d \in \bar{D},m \in M,  && \label{Cons2} \\  
	& \sum_{i \in V}x_{di}^{m} \leq q. y_{d}^{m} \hspace{1.65cm} \forall   d \in \bar{D},m \in M,  &&   \label{Cons3} \\ 
	& \sum_{i \in V}x_{id_{0}}^{m} =1,\sum_{i \in V}x_{d_{0}i}^{m} = 1 \hspace{0.65cm} \forall m \in M,&&   \label{Cons4} \\
	& \sum_{i \in V}\sum_{m \in M}x_{ij}^{m} = 1,\sum_{i \in V}\sum_{m \in M}x_{ji}^{m} = 1  \hspace{0.35cm} \forall   j \in T,&&   \label{Cons5,6} \\
	&x^{m}\big(\sigma^{+}(S)\big)\geq y_{d}^{m} \nonumber\\
	& \hspace{0.5cm} \forall   d \in S \cap \bar{D}, S \subset V \setminus \{d_{0}\}:S \cap \bar{D}\neq \emptyset, m \in M, &&  \label{Cons7} \\
	&\sum_{j \in V}z_{ij}^{m}-\sum_{j \in V}z_{ji}^{m}=\sum_{j \in V}f_{ij}^{m}x_{ij}^{m} \hspace{0.6cm} \forall   i \in T , m \in M, &&  \label{Cons8} \\
	&z_{ij}^{m}\leq F_{m}x_{ij}^{m} \hspace{1.65cm} \forall   (i,j)\in E , m \in M, &&  \label{Cons9} \\
	&z_{di}^{m}= f_{di}^{m}x_{di}^{m} \hspace{1.65cm} \forall   i \in T, d\in D , m \in M, &&  \label{Cons10} \\
	& x_{ij}^{m}\in \{0,1\},z_{ij}^{m}\geq 0  \hspace{0.2cm} \forall (i,j) \in E, m \in M, &&\label{Cons11}\\
	& y_{d}^{m}\in \{0,1\}  \hspace{0.2cm}  \forall d \in \bar{D}, m \in M.&&\label{Cons12}
	\end{alignat}
	 The objective function \eqref{obj} minimizes the maximum distance traveled by any EV. Constraints \eqref{Cons1} represents the upper-bound for the traveled distance of any EV $m$ using the continuous variable $w$. Constraints \eqref{Cons2} ensure the in-degree and out-degree of any EV using CS $d$ to be equal. Constraints \eqref{Cons3} force $y^{m}_{d}=1$ if EV $m$ visits CS $d$. Constraints \eqref{Cons4} ensure that EVs start and end their trip from the base depot $d_0$. Constraints \eqref{Cons5,6} guarantee that each target should be visited exactly once and by one EV. Connectivity of any feasible solution is guaranteed by constraints \eqref{Cons7}. Constraints \eqref{Cons8} introduce the flow variable $z_{ij}^{m}$ for each edge $(i,j) \in E$ and also removes the sub-tours. Constraints \eqref{Cons9} and \eqref{Cons10} ensure that no EV runs out of charge during its trip. Finally, constraints \eqref{Cons11} and \eqref{Cons12} define the restrictions for the decision variables.
	
	\begin{proposition} \label{prop:Con_Replacement}
		The following constraints are valid LP-relaxation of constraints \eqref{Cons3} and \eqref{Cons12}:
		\begin{alignat}{3}
		& {x_{di}^{m}} \leq y_{d}^{m} \hspace{0.5cm} \forall i \in T \cup  \{d_{0}\} \  d \in \bar{D},m \in M,  && \label{Cons14} \\  
		& 0\leq y_{d}^{m}\leq 1 \hspace{0.5cm} \forall  d \in \bar{D}, m \in M.  && \label{Cons15} 
		\end{alignat}
	\end{proposition}
	\begin{proof}
		Since constraints (\ref{Cons14}) ensure that EV $m$ can use a charging station $d$ only if $y_{d}^{m}=1$, hence they are valid constraints for MEVRP. Besides, they force the value of $y_{d}^{m}$ to be either 0 or 1 as $x_{di}^{m} \in \{0,1\}$ for any $i \in T \cup \{0\}$. Therefore, the binary restriction on variables $y_{d}^{m}$ are relaxed in (\ref{Cons15}). The proof is based on \cite{sundar2016exact}
	\end{proof}
	\vspace{-3mm}
	\section{Methodology and Algorithm Development}\label{Method}
	\subsection{Branch and Cut Algorithm}\label{B&C}
	In this section, we describe the main components of a branch-and-cut algorithm used to optimally solve the formulation presented in Section \ref{formula}. Majority of previous studies use a set of dummy binary variables for CSs to maintain the connectivity of the EV tours (\cite{schneider2014electric,zhang2018meta,erdougan2012green}), and the connectivities should be determined prior to the optimization by the users. This type of formulation can significantly increase the computational effort due to poor LP-relaxation. However, our formulation contains constraints \eqref{Cons7} to guarantee the connectivity of any feasible solution without using dummy binary variables. But, the number of such constraints is exponential, and it may not be computationally efficient to consider all these constraints in advance while using an off-the-shelf solver. To address this issue, we relax the constraints \eqref{Cons7} from the formulation. Whenever there is a feasible integer solution, we check if any of the constraints \eqref{Cons7} are violated. If so, we add the corresponding constraint and continue solving the problem. It has been observed that this process is computationally efficient for a variety of the VRP problems \cite{sundar2017path,venkatachalam2019two}.
	
	Now, we describe the details about the algorithm used to find a constraint \eqref{Cons7} that is violated for a given integer feasible solution to the relaxed problem. For every EV $m \in M$, a violated constraint \eqref{Cons7} can be denoted by a subset of vertices $S \subset V \setminus \{d_0\}$ such that $S \cap \bar{D} \neq \emptyset $; and
	$x^{m}(\sigma^{+}(S)) < y_{d}^{m}$ for every $d \in S \cap \bar{D}$ and for every EV. We construct an auxiliary graph $G^{'}=(V^{'},E^{'})$ for any feasible solution where $V^{'}= T \cup {d_0} \cup \{ d \in \bar{D} : y_{d}^{m}=1 \}$, and $E^{'}=\{(i,j) \in E : x_{ij}^{m}=1\}$. We then find the strongly connected components (SCC) of this graph. Every SCC sub-graph which does not contain the base depot $d_0$ violates the constraint \eqref{Cons7}. Hence, we add all these constraints for any feasible integer solution until we reach optimality. To implement this algorithm within the branch and cut framework, we use the \textit{Callback} feature provided by most of the commercial solvers like Gurobi \cite{gurobi}. Although the branch-and-cut can find an optimal solution, the MEVRP is an extension of VRP problem and it is NP-hard \cite{erdougan2012green}. Thus, to circumvent computational challenges for large-scale instances, we develop a hybrid heuristic method in the next section.
	\vspace{-3mm}
	\subsection{HHA Method}
	To solve MEVRP using a heuristic algorithm, we have three major challenges: 1) assigning targets to the EVs; 2) finding the best route for each EV; and 3) maintaining feasibility such that each EV does not run out of fuel. Each of these is complex, hence a naive heuristic may not be sufficient. Therefore, we implemented hybrid heuristic algorithm by integrating an linear programming model, VNS, genetic algorithm and multiple heuristics in three different phases to produce high-quality solutions.\par
	
	The flowchart of HHA for MEVRP is shown in Fig. \ref{Diagram}. The algorithm initializes by computing a modified traveling cost matrix for each EV to consider the charging limitation of EVs. Subsequently, a linear programming (LP) relaxation of an assignment problem is solved to assign the targets to the EVs. Then, an optimal or a sub-optimal travelling salesman problem (TSP) tour for the EVs is determined by the Lin-Kernigan-Helgaun heuristic \cite{helsgaun2000effective}. In the next step, the feasibility of each route in terms of range limitation before charging is checked. On every infeasible route, a novel heuristic is applied to find CSs for recharging and the distance traveled by each EV is recalculated, and an initial solution is obtained. For a pool of high-quality solutions, an iterative VNS procedure is implemented, and uses the initial solution as the incumbent. Three different insertions and a swap operation are used to improve the incumbent solution. At each iteration, a new solution is generated and compared to the incumbent. If the new solution is better than the incumbent, it is considered as the `new' incumbent and added to a pool. Also, if a new solution is not better but has a cost relatively closer compared to the incumbent, it is considered as a potentially good solution and added to the pool. This process is repeated until either no improvement is found or the maximum number of iterations is reached. 
	
	Once the pool is filled with high-quality solutions using VNS, GA parameters such as the iteration number, population size, crossover rate, mutation rate, stopping criteria are initialized. The solutions in the pool represent GA chromosomes, and the fitness value of chromosomes is considered as the maximum distance traveled by any EV. The chromosomes are sorted based on their fitness value and the ones with the higher fitness values are eliminated based on fixed population size. During the improvement phase, through a roulette wheel selection operation, some chromosomes are selected for the GA operations. The GA operations such as crossover and mutation are performed to generate new solutions (offsprings). The routing and feasibility check are performed again on the offsprings. The fitness value of the feasible offsprings is measured and compared to other chromosomes. These steps constitute an iteration, and then the roulette wheel selection is applied again to begin the next iteration. The HHA is terminated whenever a stopping criterion is met. In the post improvement, a heuristic is used on some of the best chromosomes to further improve the solution from the improvement phase. Steps of HHA are explained in the subsequent sections.
	
	\begin{figure*}
		\centering
		\includegraphics[scale=0.4]{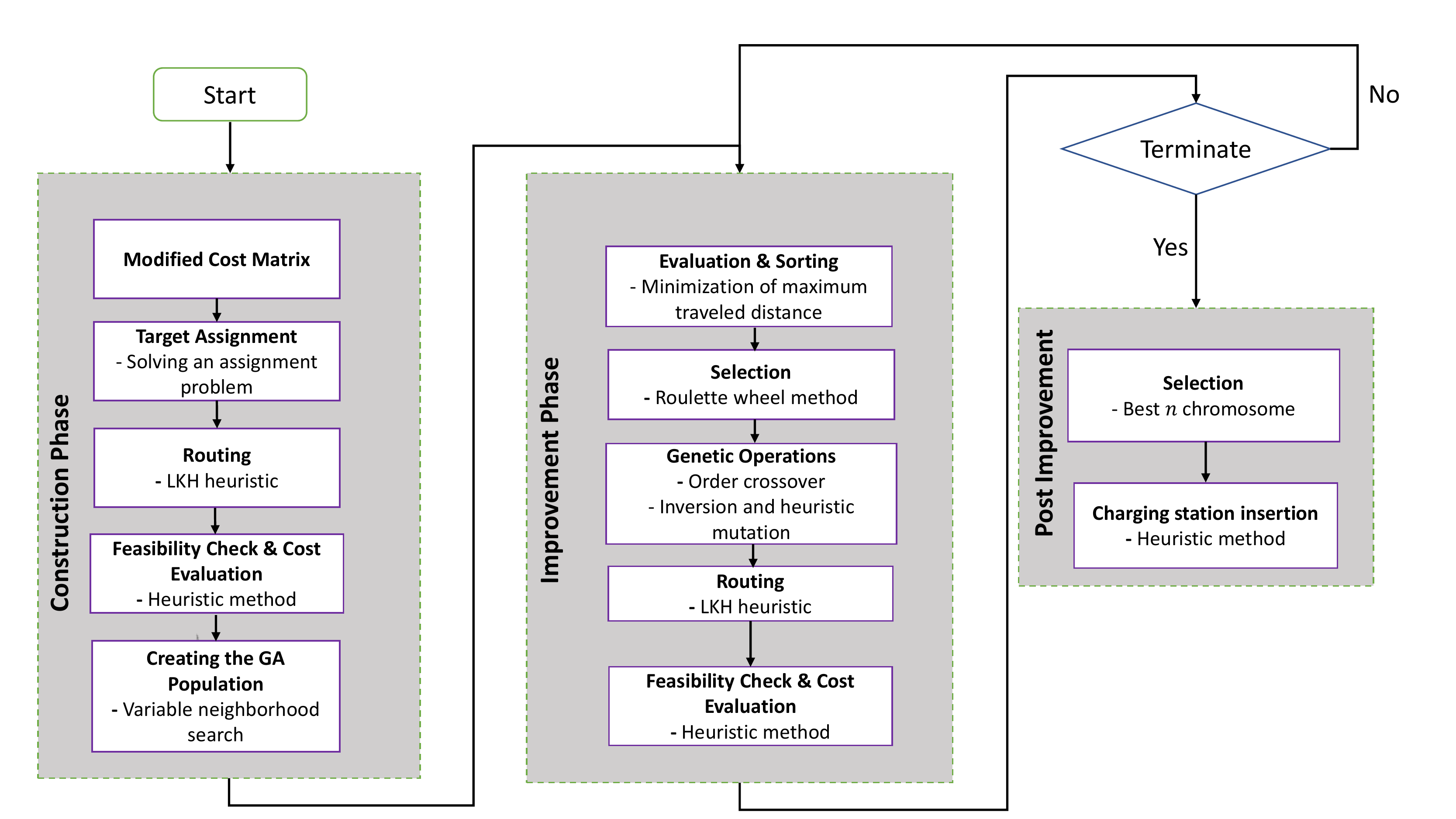}
		\vspace{-1mm}
		\caption {Flowchart of the HHA approach including the three main phases: construction, improvement and post improvement.}
		\label{Diagram}
	\end{figure*}
	
	\vspace{2mm}
	\subsubsection{Construction Phase}
	The goal of construction phase is to produce high-quality feasible solutions as initial solutions for the MEVRP. A series of steps are followed, and the details are elaborated in the following sections.
	\paragraph{Path Representation}
	To encode the solution of MEVRP problem, we use path representation in a way that targets are listed based on the order in which they are visited. Suppose that there are 10 targets and numbered from $t_5$ to $t_{14}$. In order to form a chromosome, we generate a path where targets are randomly placed in each gene of the chromosome. A sample chromosome of MEVRP problem is as follows:
	\begin{center}
		\begin{tabular}{ |c|c|c|c|c|c|c|c|c|c| } 
			\hline
			$t_{10}$ & $t_{7}$ & $t_{11}$ & $t_{12}$ & $t_{9}$ & $t_{5}$ & $t_{6}$ & $t_{8}$ & $t_{13}$& $t_{14}$ \\ 
			\hline
		\end{tabular}
		\label{chr}
	\end{center}
	Note that each chromosome contains $|M|$ string, each related to an EV.
	\paragraph{Modified Cost Matrix}
	In this step, considering the range limitation of EVs, a new traveling cost based on  \cite{khuller2007fill} is computed. The new cost matrix considers the additional charge that an EV may need to visit the available CSs as it traverses between two targets. The maximum charge remaining at the EV's battery when it visits a target $i$ is denoted as $F-f^{min}_{d_{i}}$, where $f^{min}_{d_{i}}$ denotes the amount of fuel an EV requires to reach a CS with the minimum distance from target $i$. This ensures that in any feasible tour, an EV must have an option of recharging at the nearest CS to continue visiting other targets when it reaches a target $i$. Hence, an EV can directly travel from target $i$ to target $j$ if and only if $f_{ij} \leq F-f^{min}_{d_{i}}-f^{min}_{d_{j}}$. If the EV is unable to directly travel from target $i$ to $j$, an auxiliary graph $G^{*}=(V^*,E^*)$ is created where $V^*= \bar{D}\cup{\{i,j\}}$. Any edge that can satisfy the fuel constraint will be added to the auxiliary graph. The following three sets of edges to the graph: 
	
	\begin{equation}\nonumber
	E^{*}=:
	\begin{cases}
	(i,d): & \text{if}\ f_{id}\leq  F-f^{min}_{d_{i}}, \forall d \in \bar{D}, \\
	\cup (d,j): & \text{if}\ f_{dj}\leq  F-f^{min}_{d_{j}}, \forall d \in \bar{D}, \\
	\cup (d_k,d_{k^{'}}): & \text{if}\ f_{d_kd_{k^{'}}}\leq  F, \forall d_k,d_{k^{'}} \in \bar{D}. 
	\end{cases}
	\end{equation}
	
	For every EV, the cost (length) of the edges in the graph $E^{*}$ is calculated by finding the shortest path between any two nodes using Dijkstra's Algorithm \cite{dijkstra1959note}. The modified cost matrix is created by replacing the new computed costs with old costs in the original cost matrix. Let the modified cost vector be denoted as $c'$.
	\vspace{-1mm}
	\paragraph{Target Assignment and Routing}
	To initially assign targets to the EVs, we use a LP-based Load Balancing Heuristic (LLBH) as  suggested in \cite{carlsson2009solving} with some modifications. The LLBH assumes that the distance traveled by the EVs should almost be the same, and designed for multi-depot problems. Hence, in a network where targets are uniformly distributed, EVs visit nearly the same number of targets. For the sake of compatibility, we perturb the location of EV in the base depot to create multiple depots. We uniformly place the EVs on a small circle around the base depot as shown in Fig. \ref{pert}. Given a set of $K$ base depots, indexed by $i$ where one EV is stationed at each of them, and set of $T$ targets, indexed by $j$, the LLBH solves the relaxation of the following assignment problem where $u_{ij}$ is 1 if target $j$ is assigned to the depot $i$, and 0 otherwise. Also, $c'_{ij}$ indicates the modified cost matrix.
	\begin{figure}[!htbp]
		\centering
		\includegraphics[scale=0.3]{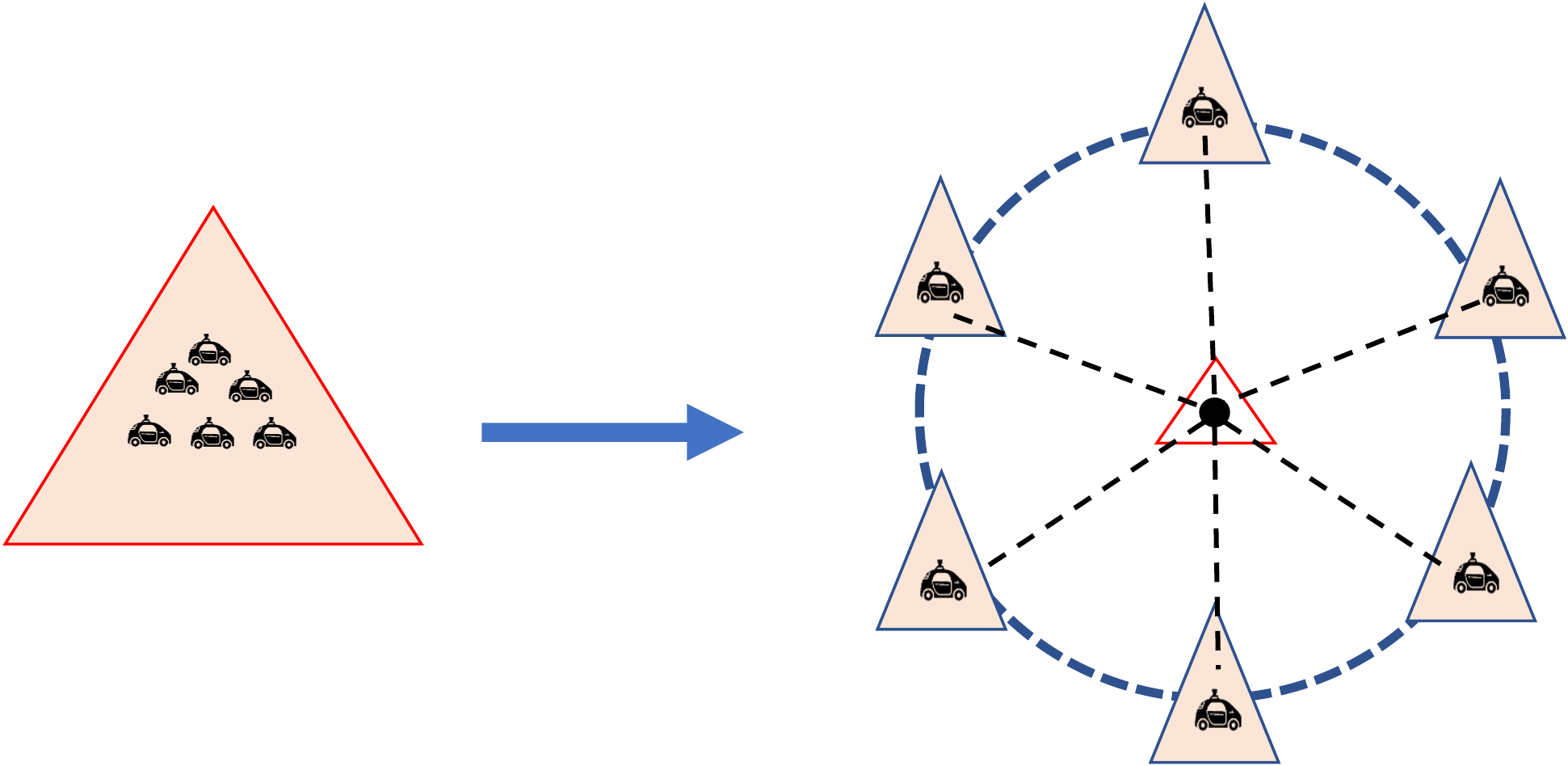}
		\vspace{-1mm}
		\captionsetup{justification=centering}
		\caption{Conversion of single depot to multi-depots using perturbation method.}
		\label{pert}
	\end{figure}
	\begin{alignat}{3}
	& \text{\textbf{P}: Min} \sum_{i \in K} \sum_{j \in T} c'_{ij}u_{ij}  && \label{obj_2} \\ 
	& \text{s.t. } && \nonumber \\
	& \sum_{i \in K}{u_{ij}}= 1 \hspace{0.6cm} \forall   j \in T,   && \label{Cons20} \\  
	& \sum_{j \in T}{u_{ij}}= \floor{\frac{|T|}{K}} \hspace{0.6cm} \forall   i \in K,   && \label{Cons21} \\  
	& u_{ij}\in \{0,1\}  \hspace{0.2cm}  \forall i \in K, j \in T. &&\label{Cons22}
	\end{alignat}
	
	Constraints (\ref{Cons20}) assign targets to the `$K$' copies of depots made around the base depot $d_0$ where each depot contains one EV. Constraints (\ref{Cons21}) determine the number of targets that each EV should visit. For the cases where $\frac{|T|}{K}$ is fractional, as per \cite{carlsson2009solving}, we assume that $|T|= pK+r$, with $p,r \in \mathbb{Z^{+}}$ and $r$ is residue. The extra $r$ targets are assigned to the EVs based on a saving technique described in the next section. After the initial assignment of targets to EVs by solving \textbf{P}, the auxiliary depots are removed. Then, we add the base depot ($d_0$) to the beginning and end of each route which result in a group of strings.
	Now, we solve a single EV routing problem for every set of targets assigned to each EV including the base depot using the Lin-Kernigan-Helgaun (LKH). LKH is considered as one of the best algorithms for solving single vehicle routing problem without recharging constraints.
	
	\paragraph{Feasibility Check} \label{FC} 
	The route found by LKH heuristic for each EV could be infeasible due to the exclusion of range constraint \eqref{Cons10}. Unlike the other approaches in literature where the infeasible routes are dismissed or penalized, a novel feasibility approach is used to convert them to feasible routes. This method is specifically beneficial when number of CSs are limited. Hence, this step is performed to attain feasibility for the tours. We calculate the charge consumption of the EVs as they traverse along their route. Whenever the charge consumption exceeds the battery capacity, we stop at the last target (e.g. $t_i$) that EV visited before recharging. At this step, a CS should be selected for recharging. The algorithm developed by  \cite{zhang2018meta} considers the minimum of $dist(t_i,d_{k})$ $\forall k \in \bar{D}$, where `$dist$' represent distances based on the parameter $f$. However, in this work, we choose the minimum of $dist(t_{i+1},d_{k})$ $\forall k \in \bar{D} $ . The advantage of this approach is that the EV will have more charge to complete the rest of the trip, and also has less necessity for recharging. Consequently, this may reduce the overall distance. However, if the EV does not have sufficient charge to reach $d_k$, then the minimum of $dist(t_{i},d_{k^{'}})$ $\forall k^{'} \in \bar{D}\setminus \{k\}$  is selected for recharging. If visiting $d_{k^{'}}$ is also not possible, then ``backward" move is performed to reach the previous edge ($t_{i-1}$, $t_i$) and perform the similar evaluation. A route is infeasible if one of the following circumstances occurs: 1) backward moves resulted in visiting the base depot; 2) all the available CSs in one edge are visited and yet it is impossible to move to the next edge; or 3) exiting from a target is impossible due to insufficient charge, and also the closet CS to the target is visited in the previous edge.\\
	\indent The feasibility check procedure is given in Algorithm \ref{f_c}. In this algorithm, a MEVRP route is defined as a sequence of targets ($t_0,t_1,t_2,...,t_{i-1},t_i,t_{i+1},...,t_{n+1}$) where the $t_0$ and $t_{n+1}$ represent the depot. We also define $m_{(t_i,t_{i+1})}=1$ if the edge $(i,i+1)$ is traversed by an EV. We consider $c^{min}_{d_{i}}$ for the CS $d$ which has the minimum distance to the target $i$. Also,  $d_{(v_i,v_{i+1})}^{k}=1$ if the eligible CS $d^{k},  \forall k \in \bar{D}$ is visited by the EV, and location of the EV is referred as $loc$.
\begin{algorithm}[!htbp]
\caption{: Feasibility Check}
	\begin{algorithmic}
    	\State \textbf{Initialization:}
    	\State \hskip1.5em Set: \textit{Status} $\gets$ Feasible, \textit{BM} $\gets$ False, \textit{FM} $\gets$ True, 
    	\State \hskip3.4em  \textit{F} $\gets$ Battery Capacity, \textit{i $\gets$ $0$}, \textit{$loc$ $\gets$ $t_i$}
    	\State \textbf{while}  \textit{any $m_{(t_i,t_{i+1})}\neq 1$ }:
    	\State \hskip2em \textbf{if} \textit{Status = Infeasible}
    	\State \hskip3em  \textit{break} 
    	\State \hskip2em \textbf{else} \textit{do FM}
    	\State \hskip2em $i$++
    	\State \hskip2em \textbf{while} \textit{FM = False}
    	\State \hskip3.8em  \textit{do BM} 
    	\State \textbf{Forward Move (FM):}
    	\State \hskip1em \textbf{if:} \textit{$z \geq f_{(t_i,t_{i+1})}$}:
    	\State \hskip1.5em  loc $\gets$  $t_{i+1}$,  $z-=f_{(t_i,t_{i+1})}$ , $m_{(t_i,t_{i+1})}= 1$ \State \hskip2em \textbf{Return} True
    	\State \hskip1em \textbf{elif:} \textit{$z \geq  f_{(t_i,c^{min}_{d_{i+1}})}$ and $F \geq f_{(c^{min}_{d_{i+1}},t_i+1)}$  and  \State \hskip2.75em $m_{(t_i,t_{i+1})}\neq 1$}:
    	\State \hskip2em  loc $\gets$ $t_{i+1}$,  $z= F-f_{(c^{min}_{d_{i+1}},j)}$ , $m_{(t_i,t_{i+1})}= 1$,
    	\State \hskip3.5em  $d_{(t_i,t_{i+1})}^{k}=1$ \State \hskip2em \textbf{Return}= True
    	\State \hskip1em \textbf{elif:} \textit{$z \geq f_{(i,c^{min}_{d_i})}$ and $F$ $\geq$ $f_{(t_i,c^{min}_{d_i})}$  and $m_{(t_i,t_{i+1})}\neq 1$}:
    	\State \hskip3.5em  loc $\gets$ $t_{i+1}$,  $z= F-f_{(c^{min}_{d_i},t_{i+1})}$ , $m_{(t_i,t_{i+1})}= 1$,
    	\State \hskip3.5em $d_{(t_i,t_{i+1})}^{k}=1$
    	\State \hskip2em  \textbf{Return} True
    	\State \hskip1em \textbf{else:} \textbf{Return} False
    	\State \textbf{Backward Move (BM):}
    	\State \hskip3.5em  \textbf{if:} loc = $t_0$ \textbf{or} $d_{(t_i,t_{i+1})}^{k}=1, \forall k \in \bar{D}$
    	\State \hskip4.5em  \textit{Status} $\gets$ Infeasible
    	\State \hskip4.5em  \textbf{Stop} 
    	\State \hskip3.5em  \textbf{else:}
    	\State \hskip4.5em   loc $\gets$  $t_{i-1}$, $z+=f_{(t_{i-1},t_{i})}$, $m_{(t_{i-1},t_{i})}= 0$,
    	\State \hskip4.5em    $d_{(t_{i-1},t_{i})}^{k}=0$, 
    	\State \hskip4.5em   do \textbf{FM}
    \end{algorithmic}
    \label{f_c}
\end{algorithm}
    \indent The feasibility procedure is applied on every route to create a feasible solution along with its corresponding cost for the MEVRP problem. This solution $x$ is considered as the incumbent solution $x_{inc}=x$.
	\paragraph{Variable Neighborhood Search (VNS)}
	
	To create a pool of high-quality solutions for the GA, a VNS-based heuristic is used. The VNS framework is based on a systematic change of a neighborhood integrated with a local search \cite{hansen2010variable}. A local search starts with an initial solution $x_{inc}$ and looks for a descent direction from $x$ within a neighborhood $N(x)$. The algorithm stops if there is no further improvement. In this problem, four neighborhoods are used, including one, two, and three insertions referred as $N_{1}(x)$, $N_{2}(x)$ and $N_{3}(x)$, respectively, and swap operation as $N_{4}(x)$. In the insertion procedure, the initial solution is improved by removing a target from the longest route and inserting to another route. Based on initial solution obtained in the previous section, a new solution $x'$ from $N_{1}(x)$ is generated. In order to select the target $i$ for removal, the savings due to removing each target from the longest route is calculated as follows:
	\begin{alignat}{3}
	dist(t_j,t_i) + dist(t_i,t_k) - dist(t_j,t_k), \label{saving}
	\end{alignat}
	where $t_j$ and $t_k$ are the preceding and succeeding targets in the same route. Target with the highest savings from the longest route is inserted into another route.
	
	In the next step, the route for insertion of target is selected. The insertion cost of the selected target in other routes is calculated, and the route with the minimum insertion cost is selected. The insertion cost of each route is calculated using (\ref{saving}) where $t_i$ is inserted between any two consecutive targets $t_j$ and $t_k$ from the route. If $f(x') <f(x_{inc})$ then $x_{inc}=x'$, where $f(.)$ is the objective function represented in \eqref{obj}. Then, we apply the LKH and feasibility check on the new solution The procedure continues with a new solution from $N_{1}(x')$, otherwise, the next target with the highest savings is selected. If there is no improvement by removing any target from the longest route and inserting into another route, a new solution from $N_{2}(x)$ is used for next iteration.
	
	Whenever there is a better solution from $N_{2}(x)$, the iterations starts over from 1-insertion neighborhood. Otherwise, another target is selected for the next iteration. If there is no improvement in the 2-insertion neighborhood, $N_{3}(x)$ is used and continued as before. Apart from insertions, a swap operation $N_{4}(x)$ is also used by randomly selecting a target from the longest route and exchanging it with any target from the route with the lowest insertion cost to further improve the solution. The swap operation is activated between any insertion operations (among $N_{1}(x)$, $N_{2}(x)$, and $N_{3}(x)$) and repeated for a predetermined number of iterations ($l^{s}_{max}$). The procedure is terminated if the predefined number of non-improvement iterations is reached.
	
	The VNS algorithm procedure is summarized in Algorithm \ref{vns}.
	
	\begin{algorithm}[!htbp]
		\caption{: Variable Neighborhood Search }
		\begin{algorithmic}
			\State \textbf{Initialization:}
			\State \hskip1.5em $l,k$  $\gets$ 0, $x_{inc}$ $\gets$ $x^{0}$, \textit{$N_1$} $\gets$ True, \textit{$N_2$} $\gets$ False,
			\State \hskip1.5em   \textit{$N_3$} $\gets$ False, $i$ $\gets$ 1
			\State \textbf{while}  \textit{$k \leq l_{i}^{max}$}:
			\State \hskip2em \textbf{while}  \textit{$N_i$} = True
			\State \hskip3em $k$  $\gets$ $k+1$
			\State \hskip3em select $x' \in N_{i}(x_{inc})$
			\State \hskip3em \textbf{if}  $f(x') \leq f(x_{inc})$ :
			\State \hskip4em $x_{inc}=x'$
			, $i$  $\gets$ 1 , $l,k$  $\gets$ 0
			\State \hskip3em \textbf{if}  $\nexists$ $x \in \{x'| x' \in N_{i}(x'), f(x') \leq f(x_{inc})$\}: 
			\State \hskip4em  \textit{$N_i$} $\gets$ False
			\State \hskip4em $i$  $\gets$ $i+1$, \textit{$N_i$} $\gets$ True
			\State \hskip4em \textbf{while} $l < l_{s}^{max}$
			\State \hskip4.5em \textit{do swap} ,
			\State \hskip4.5em \textbf{if}  $f(x') \leq f(x_{inc})$ :
			\State \hskip4.5em $x_{inc}=x'$
			, $i$  $\gets$ 1 , $l,k$  $\gets$ 0
		\end{algorithmic}\label{vns}
	\end{algorithm}
	
	During the search process, every new solution is added to the GA pool. Inspired from simulated annealing, potential good solutions are stored. A solution is potentially good if $\frac{f(x')-f(x_{inc})}{f(x_{inc})} \leq s$, where $f(x')$ and ${f(x_{inc})}$ are the costs of the potentially good and incumbent solutions, and $s$ is a relatively small parameter (e.g., $s$=0.15 in our implementation). In the next step, we apply the GA operations on the pool of high quality solutions.
	
	\vspace{-3mm}
	\subsection{Improvement Phase}
	\subsubsection{Chromosome Selection}
	Chromosome selection significantly affects the GA's convergence. The roulette wheel selection was first introduced by \cite{davis1985applying} to select the chromosomes for GA operations. Each section of the roulette wheel is assigned to a chromosome based on the magnitude of its fitness value. The fitness value of each chromosome is based on objective function value given in \eqref{obj}. The fitness values of the chromosomes determine their chance of being selected.
	\subsubsection{Crossover and Mutation}
	Two approaches are used to diversify the solution space. The first approach is `order crossover' introduced in \cite{gen1996genetic} to sample and combine selections from different potential solutions. In this approach, two offsprings are generated in each iteration by choosing a sub-tour of one parent and maintaining the relative sequence of genes of another parent. The second approach is a form of mutation, which is a genetic operator that reorders some of the gene values in a chromosome from their existing state. We apply heuristic and inversion mutation proposed by \cite{gen1996genetic}. In the heuristic mutation, three genes are randomly selected from each parent, and all the possible combinations of the selected genes are generated. The best chromosome is considered as the offspring. In the inversion mutation, a sub-string of a parent chromosome is selected and used to produce an offspring.
	\subsubsection{Post Improvement}
	GA returns the best set of chromosomes after any of the predefined stopping criteria is met. In the feasibility heuristic, an insertion of a CS is performed whenever a recharging is required. A route for an EV is defined as a sequence of targets and CSs $(d_0,t_i,t_j,d_n,...,t_l,d_m,...,d_0)$ where $d_0$ denotes the base depot, $i,j,l \in V$, and $m,n \in \bar{D}$. In some instances, it is possible to obtain a better fitness value by changing the position of the CSs in chromosomes. For example, authors in \cite{schiffer2018adaptive} proposed a dynamic programming-based approach to identify optimal location for intra-route facilities. Hence, to further enhance the quality of the GA solution, we propose a heuristic which is described by an example as follows: suppose that there are three CSs ($d_1$, $d_2$, and $d_3$), and the following string is a feasible path for an EV (the visited CSs are colored with red):
	\begin{center}
		\begin{tabular}{ |c|c|c|c|c|c|c|c|} 
			\hline
			$d_0$ & $t_{8}$ & $t_{11}$ & \textcolor{red}{$d_3$}&$t_9$ & $t_5$ & \textcolor{red}{$d_2$}& $d_0$ \\ 
			\hline
		\end{tabular}
		\label{chr}
	\end{center}
	Here, a ``sub-string'' is a sequence of visits where the last node is a CS. We select the first sub-string and delete the CS from the sub-string. We then insert all the eligible CSs among the targets to generate new sub-routes. A CS is considered eligible if the EV starting from a target can reach that CS without visiting other CSs and also could reach the next target after recharging. New sub-routes are checked for battery capacity violation. The total distance traveled and the EV's remaining battery charge are stored for each feasible route. In the computational experiments, this considerably decreased the computational effort. The table (a) in Fig. \ref{Improve} shows all the new sub-routes generated by CS insertion where the feasible sub-routes are highlighted.
	Then, we delete the infeasible sub-routes and add the second sub-string. The same CS insertion procedure for the second sub-string is repeated using the information stored in the previous step. The final possible routes are shown in the table (b) in Fig. \ref{Improve}.
	\begin{figure}[!htbp]
		\centering
		\includegraphics[scale=0.25]{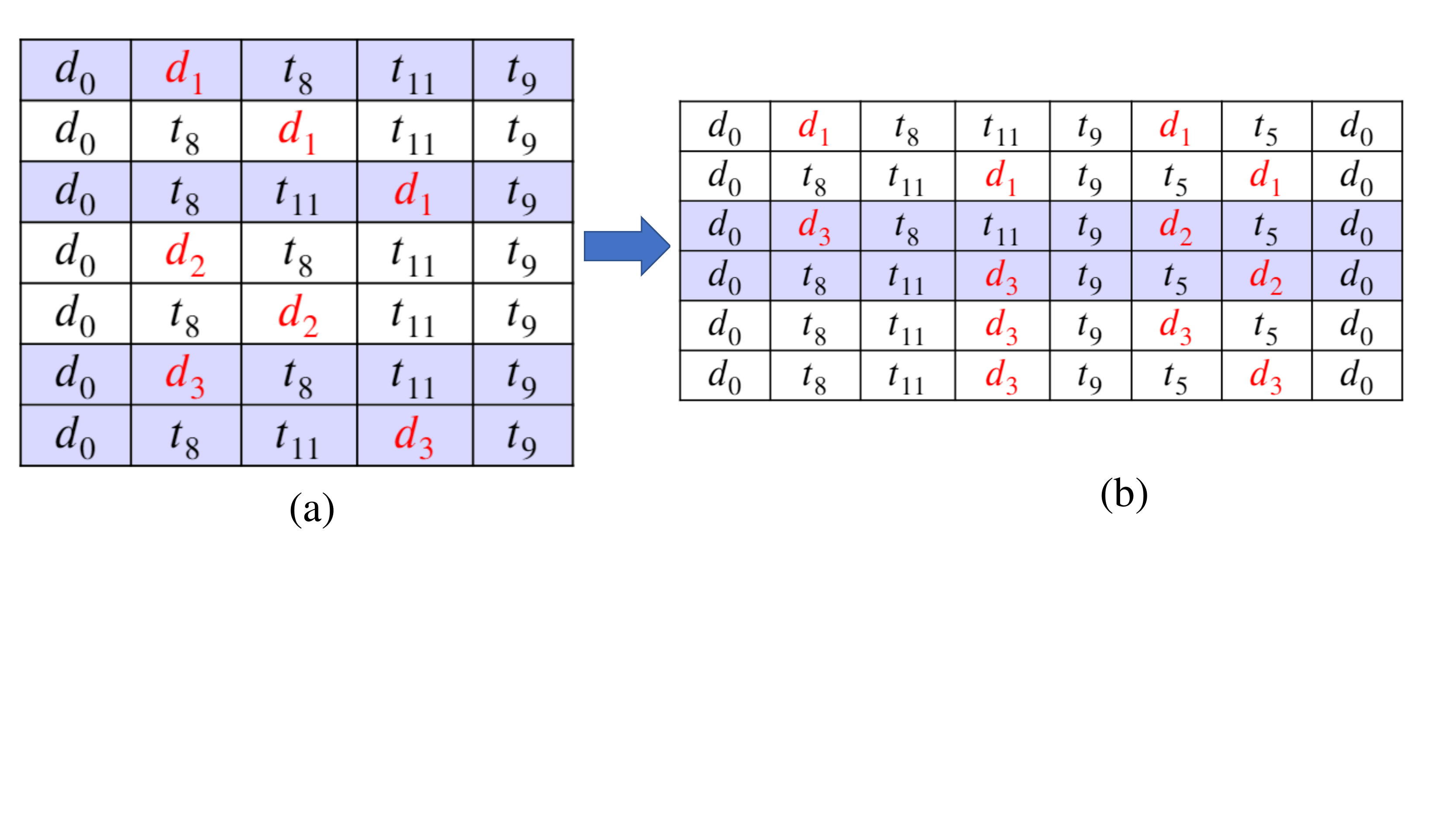}
		\vspace{-18mm}
		\captionsetup{justification=centering}
		\caption{Initial sub-tours generated in the first step of CS insertion heuristic (a), and final feasible routes (b).}
		\label{Improve}
	\end{figure}
	
The procedure is continued to visit the last node in all routes. Then the feasible route with the lowest distance is considered as the EV's total traveled distance. It should be noted that since the heuristic does not change the order of the targets, the chromosome remains the same. So, this heuristic is used only on final $n$ best chromosomes returned by GA where $n$ is a predefined user parameter (e.g., in this work we used $n$=15).
	\vspace{-4mm}
	\section{Experiments and Results}\label{ER}
	In this section, we compare the computational performance of B\&C algorithm and HHA. All the experiments were implemented in Python 3.7, and the MIP models were solved by Gurobi 9.0 \cite{gurobi} using a computer with an Intel \textregistered \, Xeon \textregistered \, CPU E5-2640, 2.60 GHz, and 80GB RAM.
	\vspace{-5mm}
	\subsection{Instance Generation}
	For the computational experiments, we selected three data sets (A, B, and P) of the capacitated vehicle routing problem developed by Augerat et al. \cite{augerat1995computational}. The data sets were modified to adapt for MEVRP by adding CSs. Additionally, a random data set was used to maintain the diversity in our experiments. The details of the data sets are as follows:
	\begin{itemize}
		\item \textit{Random instances}: Random instances were generated in a square grid of size [100, 100], and the base depot at (50, 50). The number of CSs was set to five, and the locations of the depot as well as the CSs were fixed as apriori for all the random test instances. The number of targets varies from 10 to 50 in steps of five, while their locations were uniformly distributed within the square grid. For each of the generated instances, the number of EVs in the base depot is varied from two to eight. The battery capacity of EVs was set to 100, and the energy consumption rate is considered to be 0.8.
		\item \textit{Augerat et al. instances}: The data sets reported in the study by Augerat et al. \cite{augerat1995computational} are developed for capacitated vehicle routing problems. These instances include three sets which differ in distribution and number of targets, vehicle's capacity, demand/capacity, and number of vehicles. To make the instances compatible to MEVRP, we added five CSs. The coordination of CSs are similar for all instances within each data set. Also, the capacity of vehicles is considered as the battery capacity of EVs and capacity tightness as the fuel consumption rate in our problem. Fig. \ref{data} shows the four different data sets.
	\end{itemize}
	\begin{figure}
		\centering
		\includegraphics[scale=0.5]{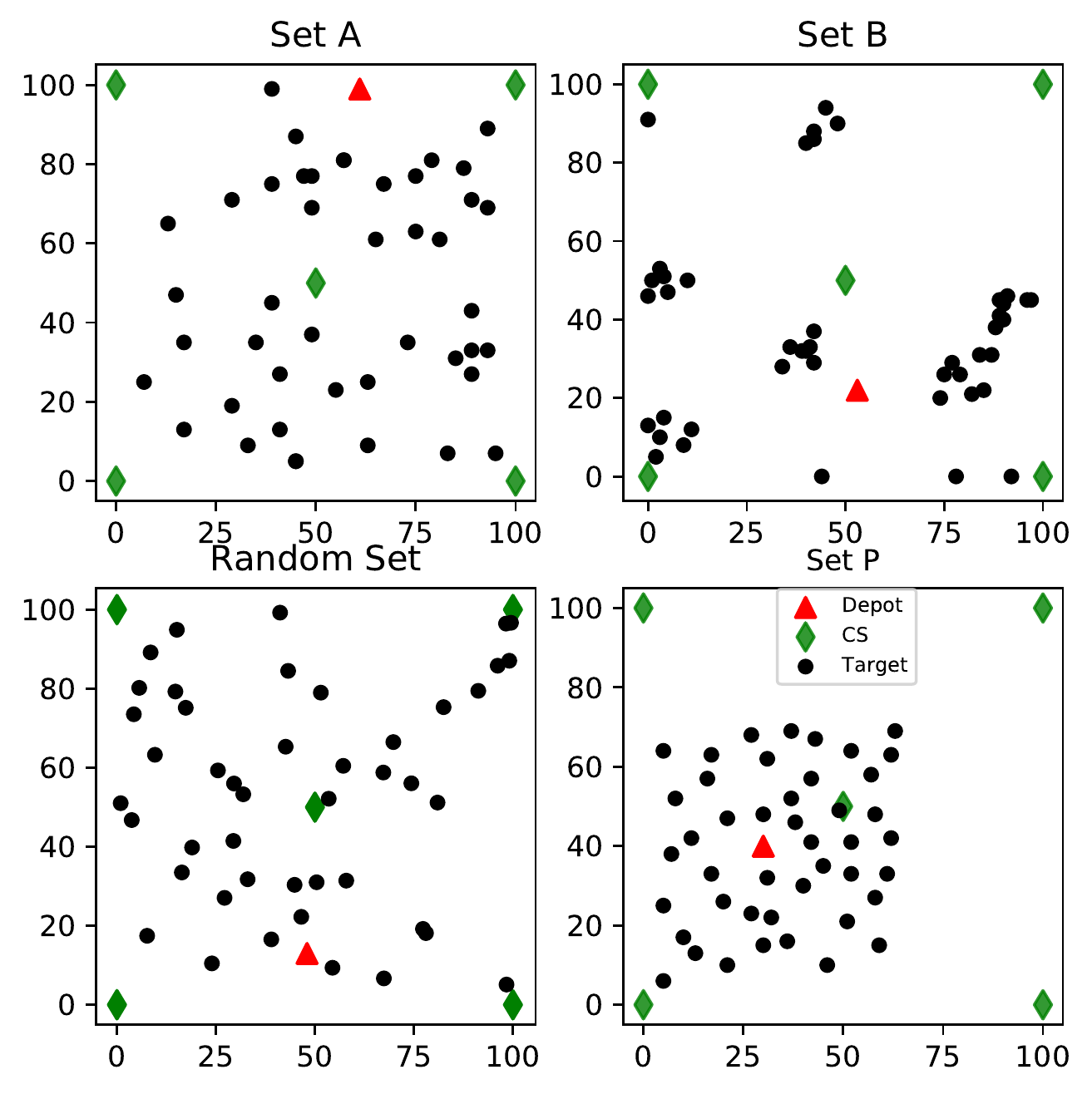}
		\vspace{-2mm}
		\caption {Three data sets based on the study in \cite{augerat1995computational} and a random set for computational experiments.}
		\label{data}
	\end{figure}
	\begin{table*}
		\caption{Characteristics of the large-scale instances for different data sets}
		\centering
		\begin{tabular}{| c | c | c | c | c |}  
			\hline
			Sets & Target range &  Vehicles range &  Consumption rate range & Battery Capacity range \\ 
			\hline
			A & [31,60] & [5,9]& [0.81,0.99] & 100 \\
			\hline
			B & [30,60] & [4,9]& [0.82,1] & 100\\
			\hline
			P& [21,59] & [5,15]& [0.88,0.99]  & [70,170]\\
			\hline
			Random & [20,50] & [4,10]& [0.8,0.8]  & 100\\
			\hline
		\end{tabular}
		
		\label{LST}
	\end{table*}
	\vspace{-3mm}
	\subsection{Parameter Tuning}
    Prior to the numerical experiments, we conducted an analysis on the parameters of maximum number of non-improving iterations for insertion ($l^{i}_{max}$),  maximum number of non-improving iterations for swaps ($l^{s}_{max}$) , GA population size ($P_{size}$) , crossover rate ($C_r$) and mutation rate ($M_r$). A three-level Taguchi design \cite{taguchi1986introduction} framework was used by considering three levels for each parameter using Minitab 19 \cite{Minitab}. A randomly chosen instance with 32  targets and five EVs  was  used for the experiments. For each combination of levels, HHA is run four times and the average of solutions for each combination is taken as response. The response values are then used as inputs, and the optimal levels of parameters are obtained.
	To demonstrate the performance of HHA algorithm for different values of parameters, additional computational experiments to investigate the impact of $l_{i}^{max}$ and $l_{s}^{max}$ parameters were performed. Random instances from each of the four sets were run for different values of parameters. The  tuned parameters are as follows: $l_{i}^{max}=40$,  $l_{s}^{max}=25$, $P_{size}=55$, $C_{r}=0.6$ and $C_{r}=0.2$
		\begin{figure}[H]
		\centering
		\includegraphics[scale=0.5]{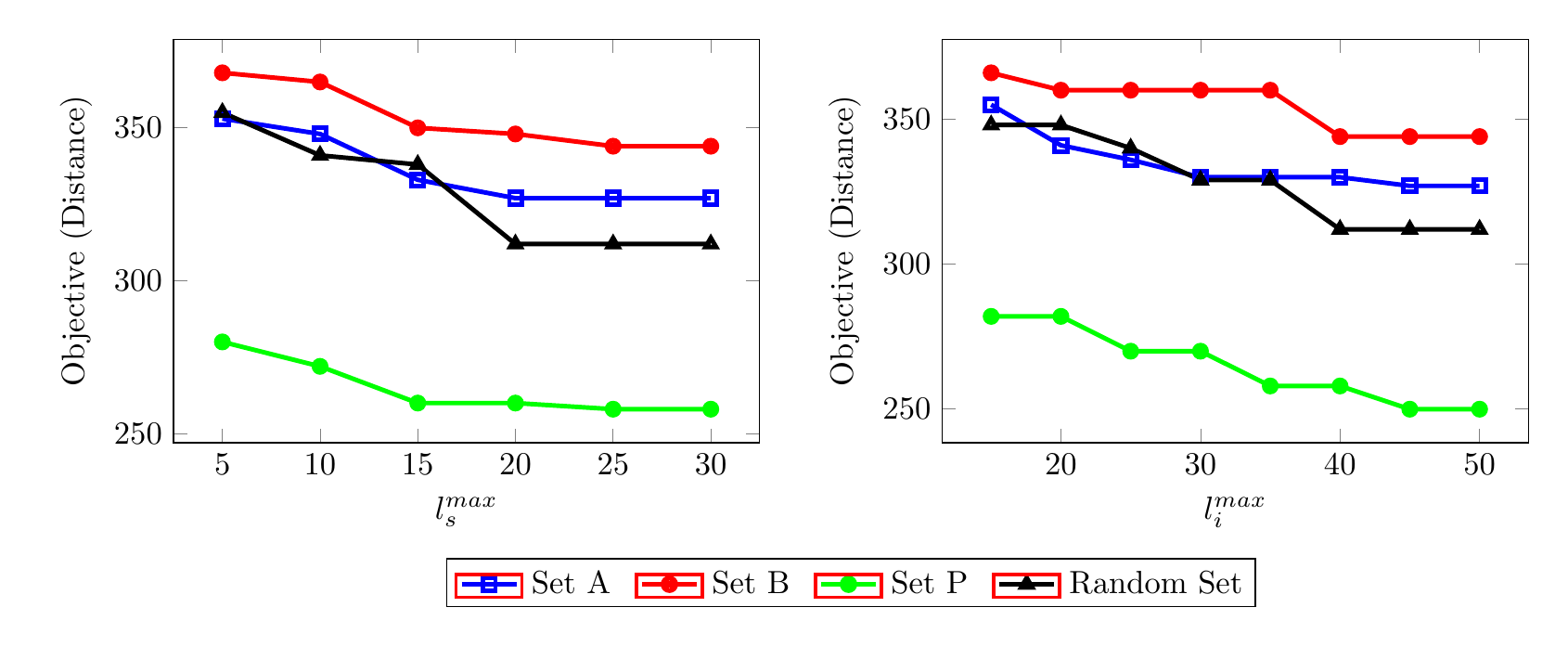}
		\vspace{-2mm}
		\captionsetup{justification=centering}
		\caption{The effect of $l_{i}^{max}$ and $l_{s}^{max}$ parameters on the objective of MEVRP}
		\label{paramss}
	\end{figure}
	\vspace{-5mm}
	\subsection{Experiment on MEVRP Instances}
	\subsubsection{Benchmark Instances}
	For benchmark instances, a subset of targets were used which are the first 10 and 15 targets with two and three EVs within each data set, respectively. We solved the benchmark instances with HHA and compared the it with optimal or near-optimal solution solved by the B\&C algorithm. Fig. \ref{boxplot} represents the differences in the objective function between HHA and B\&C for the 10-targets and 15-targets instances. The objective difference in percentage and run-time in seconds are calculated as $\frac{\text{B\&C(O)-HHA(O)}}{\text{B\&C(O)}}\times 100$, and (B\&C(t)-HHA(t)), where B\&C(O) and HHA(O) indicate objective values, and B\&C(t) and HHA(t) indicate the run-times of B\&C and HHA, respectively. In total, 71 10-targets and 71 15-targets benchmark instances were solved. For all the 10-targets instances, the B\&C algorithm was able to reach solutions within 1\% of the optimality gap. For many of the 15-targets instances from sets A, B, and Random, the B\&C algorithm was not able to find the optimal solution within the stipulated time limit of two hours. However, B\&C found the optimal solution for most of the 15-targets instances from set P. Fig. \ref{boxplot} represents the run-time and objective function differences between HHA and B\&C. For all the completed runs using B\&C, HHA reached the optimal solution or near-optimal solution within less than 2.5\% gap. For the instances where B\&C was not able to reach optimality, the HHA approach could find a solution closer or better compared to the upper bound provided by B\&C. In terms of run-time, for more than half of the 10-targets instances, HHA outperformed B\&C. Also, HHA outperformed B\&C with a much better run-time in the experiments with 15-targets instances. It should be noted that we removed the instances which reached the time-limit in Fig \ref{boxplot} using B\&C to better illustrate the run-time difference in 15-targets instances.
	The results from the benchmark instances indicate that the HHA approach is capable of providing highly efficient routes for EVs for smaller instances as well.
	\vspace{-4mm}
	\begin{figure}[!htbp]
		\centering
		\includegraphics[scale=0.25]{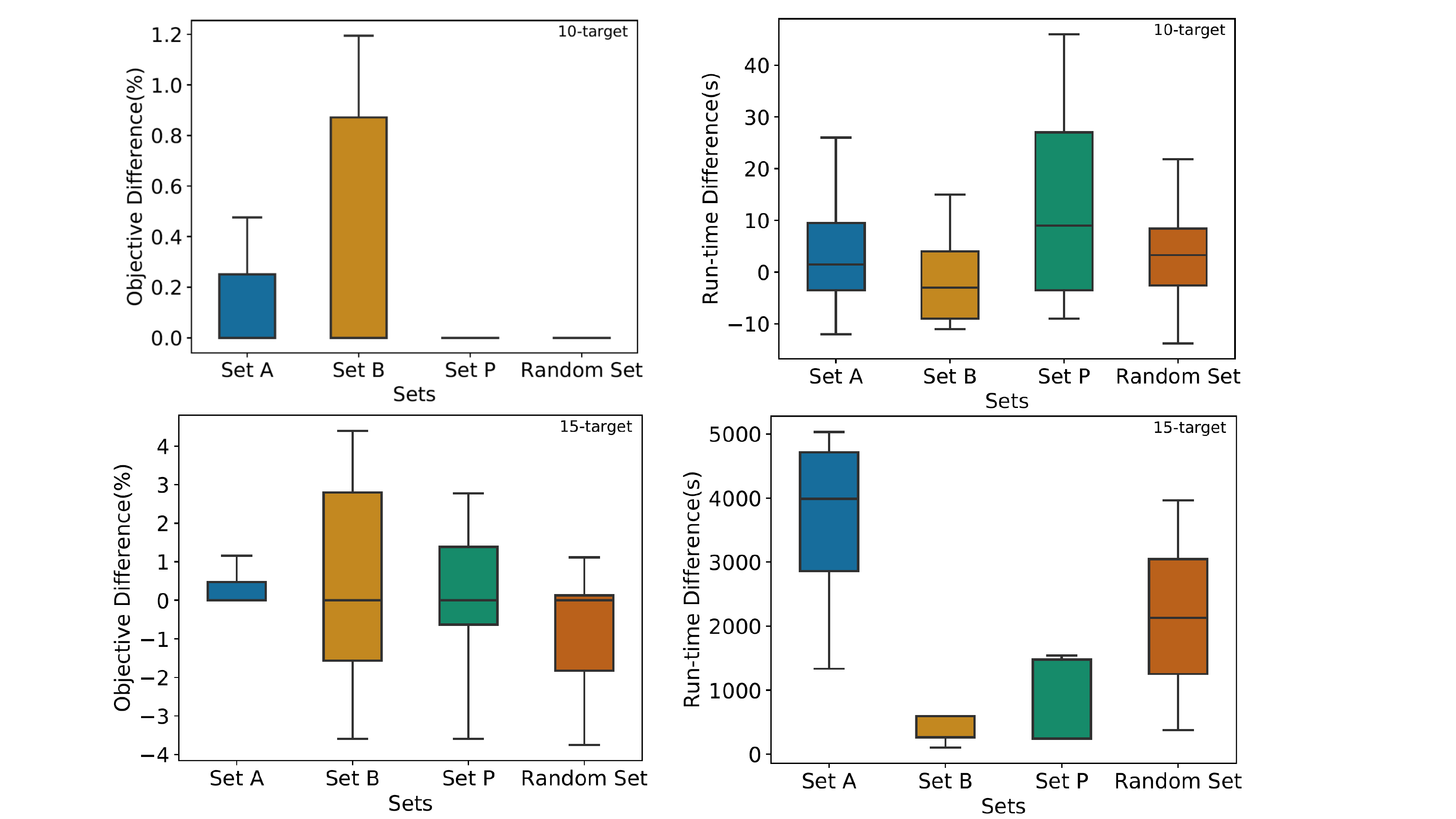}
		\vspace{-1mm}
		\captionsetup{justification=centering}
		\caption{Performance of the HHA compared to the B\&C procedure for benchmark instances with 10 and 15 targets.}
		\captionsetup{justification=centering}
		\label{boxplot}
	\end{figure}
	\subsubsection{Large-scale Instances}
	To demonstrate the performance of HHA on large scale problem sets, we considered 72 instances with the number of targets and EVs up to 60 and 15, respectively. Among these instances, 18 instances are from set A, and 17, 12, and 25 instances are from sets B, P, and Random, respectively. In 31 of the 72 total instances, B\&C procedure was not able to find any feasible solution whereas HHA was capable of providing a feasible solution for all instances in just a few seconds. Table \ref{LST} represents the comparison between the B\&C procedure and HHA for each set. The first column indicates the total number of instances, and the columns `\# of Inf(B\&C)' and `\# of Inf(HHA)' indicate the number of instances in which the B\&C procedure and HHA could not find any feasible solution in the stipulated time limit. The next two columns labeled `AG-B\&C(\%)' and `ART-HHA(\%)' specify the average gap percentage and the average run-time in seconds obtained by the B\&C and HHA approaches. The last column labeled `AOD(\%)' is the average objective difference percentage between B\&C procedure and HHA. For example, in Table \ref{LST}, for the 18 instances in set A, B\&C algorithm found a feasible solution for nine instances with an average optimality gap of 51\%, while HHA found a feasible solution for all 18 instances. Similarly, for sets B, P, and Random, B\&C procedure found a feasible solution for 4, 10 and 18 instances with an average gap of 63.7\%, 51.7\%, and 60.8\%, respectively. For all the experiments using large-scale instances, HHA outperformed B\&C procedure in terms of computational time and the quality of the solution. Fig. \ref{LS_Results} illustrates the performance of HHA compared to the B\&C algorithm where the horizontal and vertical axes indicate the number of targets and the objective for different instances. Also, the disconnected lines indicate the instances where the B\&C could not find any feasible solutions within the time limit.
	\begin{figure}
		\includegraphics[scale=0.33]{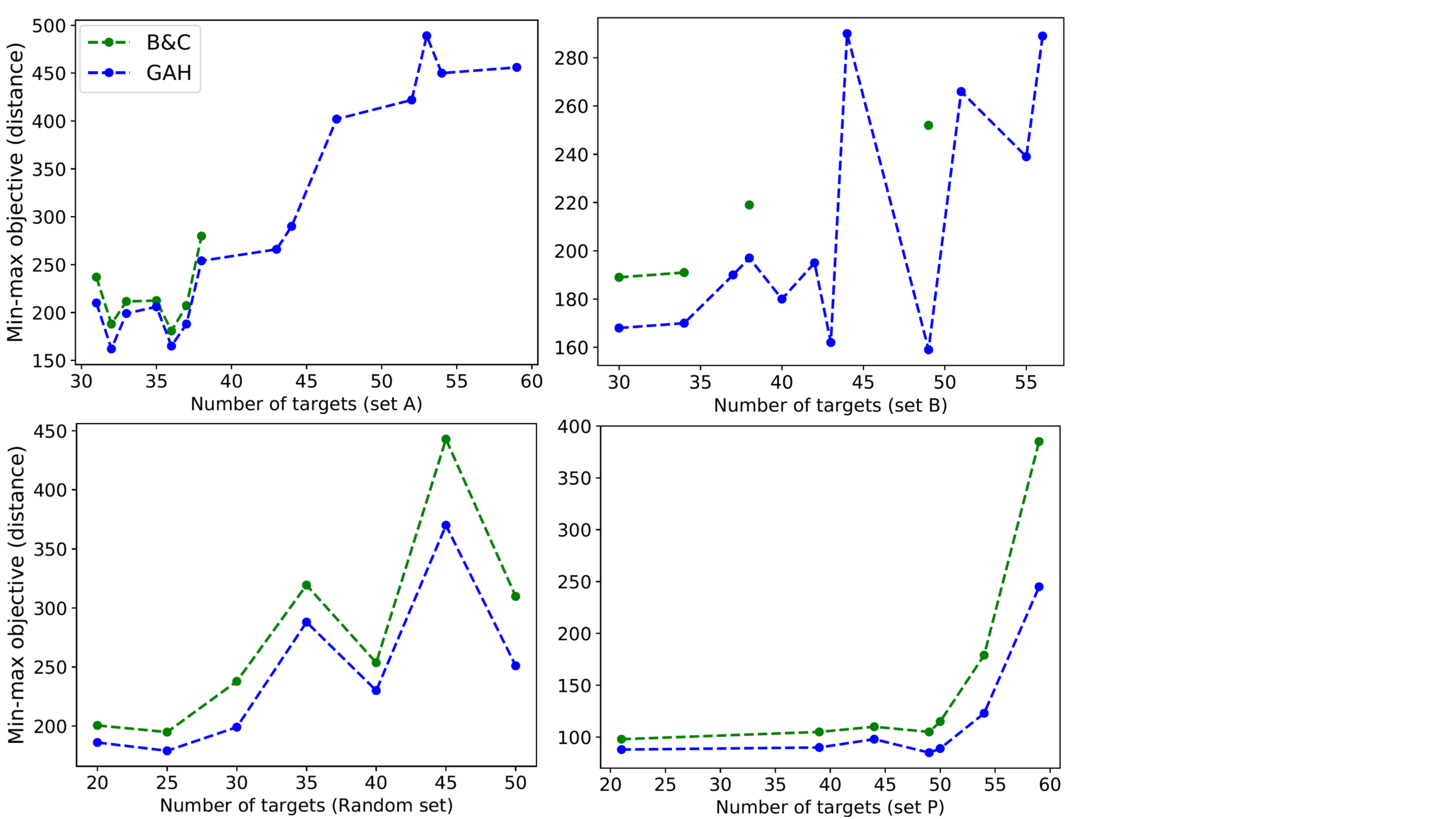}
		\vspace{-1mm}
		\captionsetup{justification=centering}
		\caption{Comparison between the performance of HHA and B\&C for different large-scale instances.}
		\label{LS_Results}
	\end{figure}
	\vspace{-6mm}
	\begin{table*}
		\caption{Comparison of the performances of HHA and B\&C procedure for large-scale instances}
		\centering
		\begin{tabular}{| c | c | c | c | c | c | c | c | c|}  
			\hline
			Sets & \# of Instances & \# of Inf(B\&C)& \# of Inf(HHA)& AG-B\&C(\%)  &ART-HHA(s) & AOD(\%)  \\ 
			\hline
			A & 18 & 9& 0& 51 & 337&12.6\\
			
			B &17& 13& 2& 62.7&  320 &13.3\\
			P &12 &2 &0 & 51.7&  254 & 22.2\\
			
			Random &25&7 &0 &60.8& 263&13.6\\
			\hline
		\end{tabular}
		\label{LST}
	\end{table*} 
	\vspace{1mm}
	\subsection{Sensitivity Analysis}
	\subsubsection{Effect of the Number of EVs}
	We analyze the effect of increasing the number of EVs on objective function and total distance traveled by all the EVs. As shown in Fig. \ref{Inc-EV}, when there is an increase in the number of EVs, in general, the maximum distance traveled by the EVs decreases. However, the total distance traveled by all the EVs increases. The instance from set B has the highest increase in the total distance traveled by the EVs as the number of EVs is increased. This result is likely due to the cluster-like distribution of the targets that occurred in that data set. When the number of EVs surpasses the number of clusters, each cluster is assigned to more than one EV which increase the total distance. Another factor that may significantly affect the total distance is the position of the depot. In the cases where the depot is far from the targets, dispatching multiple EVs could result in a longer total distance. Since the targets are closer to each other in set P, an increase in the number of EVs would not result in a significant decrease or increase in the min-max or total distance, respectively. In conclusion, decision makers should evaluate the trade-off between the maximum and total distances while considering the distribution of the targets and locations of the depots.
	\subsubsection{Effect of the Number of Charging Stations}
	Another perspective is the impact of increase in the number of CSs on the  min-max objective function. We chose an instance from each of the four data sets, and for each instance, we considered two to seven randomly located CSs. Fig. \ref{IncreaseCS} shows the effect of number of CSs on MEVRP for different instances. An increase in the number of CSs has a lower impact on sets P and B sets where the targets are confined to a relatively smaller area. Due to the scattered distribution of the targets in the random set and set A, we observe a rapid decrease in the maximum distance as the number of CSs decreases. This sort of analysis will help the logistics companies to evaluate the marginal benefits of adding more CSs based on the distribution of the customers in their network.
\vspace{-4.5mm}
	\begin{figure}
		\centering
		\includegraphics[scale=0.5]{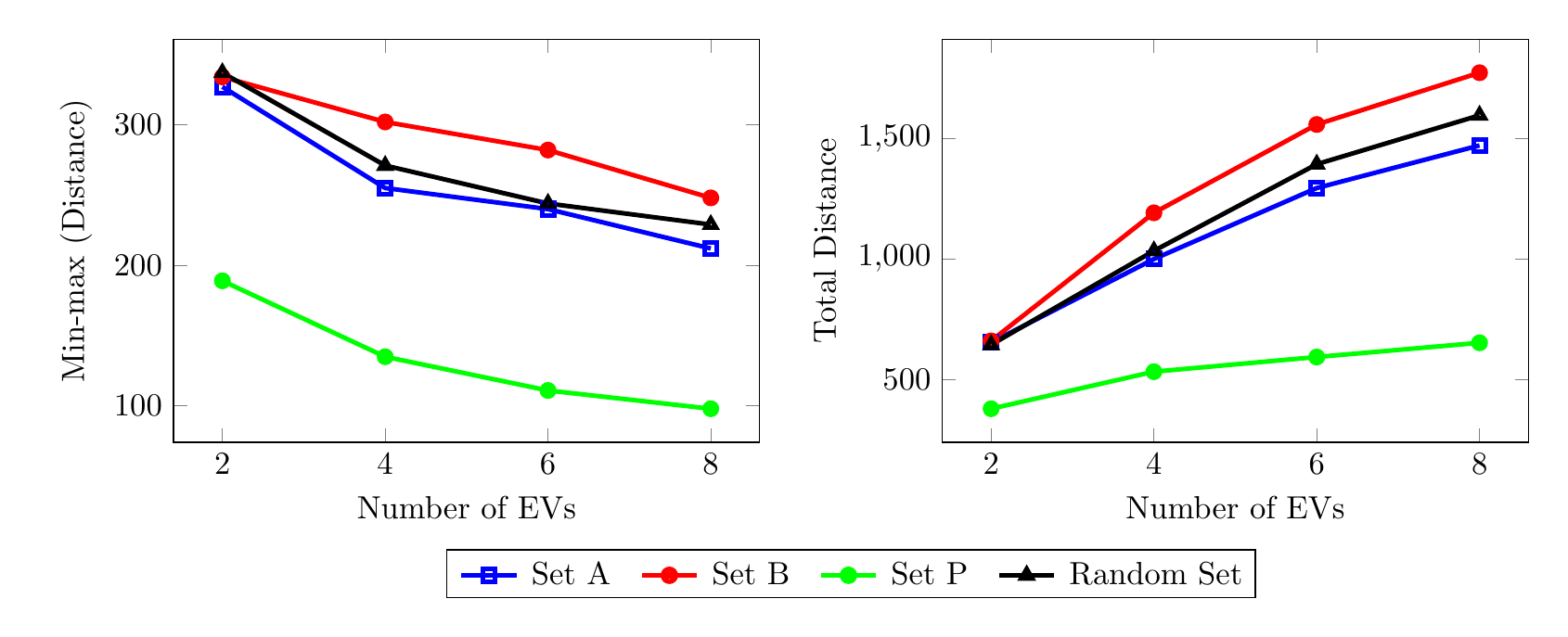}
		\vspace{-2mm}
		\captionsetup{justification=centering}
		\caption{Effect of increase in the number of EVs on the min-max and the total distance traveled by the EVs.}
		\label{Inc-EV}
	\end{figure}
	\begin{figure}
		\centering
		\includegraphics[scale=0.5]{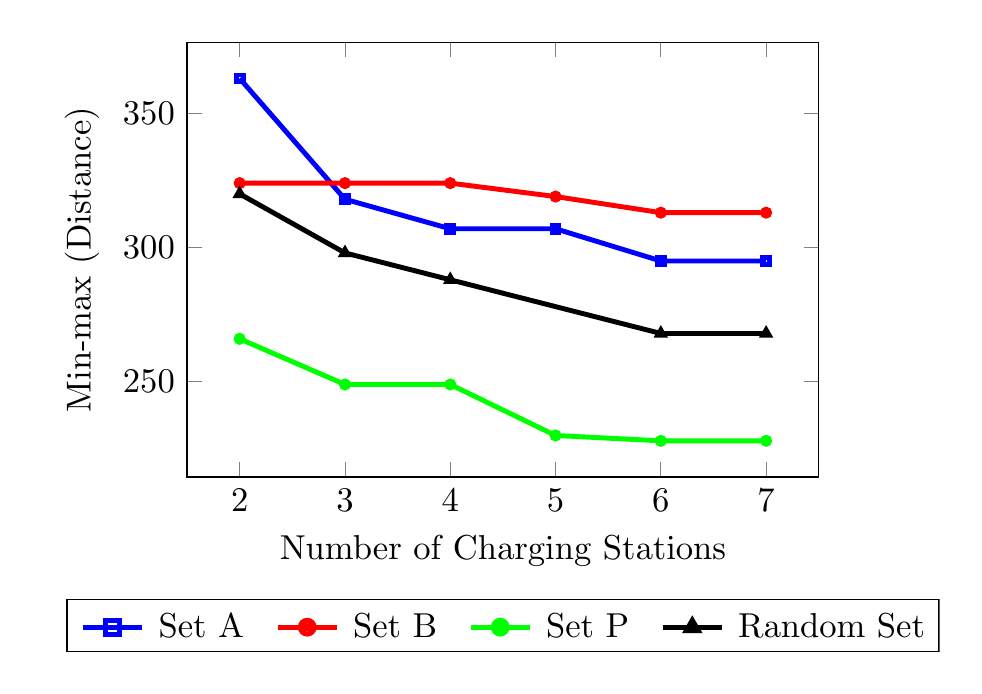}
		\vspace{-1mm}
		\caption {Effect of increase in the number of charging stations on sets A, B, P and Random.}
		\label{IncreaseCS}
	\end{figure}
	\section{Conclusion}\label{Con}
	In this research, we consider a min-max routing problem for a fleet of EVs in the presence of charging stations. We proposed an efficient mixed-integer programming formulation and a hybrid heuristic algorithm with an embedded feasibility method to solve small and large-scale instances, respectively. Numerical studies are performed on randomly generated data sets and capacitated vehicle routing problem data sets with some modifications from literature. The efficacy of the proposed methods are benchmarked using extensive computational experiments. The results indicate that HHA was able to find solutions within 1\% of the optimality gap for the 10-targets, and equal or better solutions in the 15-targets benchmark instances. Also, HHA was able to produce high-quality solutions in large-scale instances whereas the branch-and-cut procedure could not even find a feasible solution for most of the instances within the two hours time limit. The proposed methods are capable for facilitating routing and charging decisions for a fleet of EVs employed by logistics and transportation companies. The sensitivity analysis indicates that with a reasonable sacrifice in the total distance, the proposed approach can significantly decrease the travel time to visit some of the targets. In addition, this may also decrease the maintenance cost of the EVs due to fair and equitable distribution of workload among the EVs fleet. This could potentially decrease the aging or degradation of batteries in EVs. 
	Future work could include EVs with freight capacities and the targets with time windows. Another extension could be incorporating uncertainties or non-linearity in the battery consumption rate as well as the location of target locations, especially for surveillance applications. From an algorithmic perspective, use of decomposition algorithms like the column generation method can also be investigated.
	
\section{Acknowledgement}\label{Con}
The authors wish to acknowledge the technical and financial support of the Automotive Research Center (ARC) in accordance with Cooperative Agreement W56HZV-19-2-0001 U.S. Army CCDC Ground Vehicle Systems Center (GVSC) Warren, MI.

	\vspace{-3mm}
	\bibliographystyle{ieeetr}
	\bibliography{IEEE.bib}

\begin{thebibliography}{10}

\bibitem{nazaripouya2019electric}
H.~Nazaripouya, B.~Wang, and D.~Black, ``Electric vehicles and climate change:
  Additional contribution and improved economic justification,'' {\em IEEE
  Electrification Magazine}, vol.~7, no.~2, pp.~33--39, 2019.

\bibitem{solaymani2019co2}
S.~Solaymani, ``Co2 emissions patterns in 7 top carbon emitter economies: The
  case of transport sector,'' {\em Energy}, vol.~168, pp.~989--1001, 2019.

\bibitem{wu2019role}
J.~Wu, H.~Liao, J.-W. Wang, and T.~Chen, ``The role of environmental concern in
  the public acceptance of autonomous electric vehicles: A survey from china,''
  {\em Transp Res Part F Traffic Psychol Behav}, vol.~60, pp.~37--46, 2019.

\bibitem{yi2018energy}
Z.~Yi, J.~Smart, and M.~Shirk, ``Energy impact evaluation for eco-routing and
  charging of autonomous electric vehicle fleet: Ambient temperature
  consideration,'' {\em Transp. Res. Part C Emerg. Technol.}, vol.~89,
  pp.~344--363, 2018.

\bibitem{zhang2019joint}
H.~Zhang, C.~J. Sheppard, T.~E. Lipman, and S.~J. Moura, ``Joint fleet sizing
  and charging system planning for autonomous electric vehicles,'' {\em IEEE
  trans Intell Transp Syst}, 2019.

\bibitem{sundar2016exact}
K.~Sundar, S.~Venkatachalam, and S.~Rathinam, ``An exact algorithm for a
  fuel-constrained autonomous vehicle path planning problem,'' {\em arXiv
  preprint arXiv:1604.08464}, 2016.

\bibitem{schneider2014electric}
M.~Schneider, A.~Stenger, and D.~Goeke, ``The electric vehicle-routing problem
  with time windows and recharging stations,'' {\em Transportation Science},
  vol.~48, no.~4, pp.~500--520, 2014.

\bibitem{hiermann2016electric}
G.~Hiermann, J.~Puchinger, S.~Ropke, and R.~F. Hartl, ``The electric fleet size
  and mix vehicle routing problem with time windows and recharging stations,''
  {\em Eur. J. Oper. Res.}, vol.~252, no.~3, pp.~995--1018, 2016.

\bibitem{slowik2018continued}
P.~Slowik and N.~Lutsey, ``The continued transition to electric vehicles in us
  cities,'' 2018.

\bibitem{fazeli2020two}
S.~S. Fazeli, S.~Venkatachalam, R.~B. Chinnam, and A.~Murat, ``Two-stage
  stochastic choice modeling approach for electric vehicle charging station
  network design in urban communities,'' {\em IEEE trans Intell Transp Syst},
  2020.

\bibitem{kocc2016green}
{\c{C}}.~Ko{\c{c}} and I.~Karaoglan, ``The green vehicle routing problem: A
  heuristic based exact solution approach,'' {\em Appl. Soft Comput.}, vol.~39,
  pp.~154--164, 2016.

\bibitem{sundar2013algorithms}
K.~Sundar and S.~Rathinam, ``Algorithms for routing an unmanned aerial vehicle
  in the presence of refueling depots,'' {\em IEEE T AUTOM SCI ENG}, vol.~11,
  no.~1, pp.~287--294, 2013.

\bibitem{desaulniers2016exact}
G.~Desaulniers, F.~Errico, S.~Irnich, and M.~Schneider, ``Exact algorithms for
  electric vehicle-routing problems with time windows,'' {\em Operations
  Research}, vol.~64, no.~6, pp.~1388--1405, 2016.

\bibitem{vincent2017simulated}
F.~Y. Vincent, A.~P. Redi, Y.~A. Hidayat, and O.~J. Wibowo, ``A simulated
  annealing heuristic for the hybrid vehicle routing problem,'' {\em Appl. Soft
  Comput.}, vol.~53, pp.~119--132, 2017.

\bibitem{juan2014routing}
A.~A. Juan, J.~Goentzel, and T.~Bekta{\c{s}}, ``Routing fleets with multiple
  driving ranges: Is it possible to use greener fleet configurations?,'' {\em
  Appl. Soft Comput.}, vol.~21, pp.~84--94, 2014.

\bibitem{erdougan2012green}
S.~Erdo{\u{g}}an and E.~Miller-Hooks, ``A green vehicle routing problem,'' {\em
  Transportation research part E: logistics and transportation review},
  vol.~48, no.~1, pp.~100--114, 2012.

\bibitem{levy2014heuristics}
D.~Levy, K.~Sundar, and S.~Rathinam, ``Heuristics for routing heterogeneous
  unmanned vehicles with fuel constraints,'' {\em Mathematical Problems in
  Engineering}, vol.~2014, 2014.

\bibitem{lunz2012influence}
B.~Lunz, Z.~Yan, J.~B. Gerschler, and D.~U. Sauer, ``Influence of plug-in
  hybrid electric vehicle charging strategies on charging and battery
  degradation costs,'' {\em Energy Policy}, vol.~46, pp.~511--519, 2012.

\bibitem{campbell2008routing}
A.~M. Campbell, D.~Vandenbussche, and W.~Hermann, ``Routing for relief
  efforts,'' {\em Transportation Science}, vol.~42, no.~2, pp.~127--145, 2008.

\bibitem{torabbeigi2020drone}
M.~Torabbeigi, G.~J. Lim, and S.~J. Kim, ``Drone delivery scheduling
  optimization considering payload-induced battery consumption rates,'' {\em J
  INTELL ROBOT SYST}, vol.~97, no.~3, pp.~471--487, 2020.

\bibitem{zaloga2011unmanned}
S.~J. Zaloga, {\em Unmanned aerial vehicles: robotic air warfare 1917--2007}.
\newblock Bloomsbury Publishing, 2011.

\bibitem{manyam2016path}
S.~G. Manyam, D.~W. Casbeer, and K.~Sundar, ``Path planning for cooperative
  routing of air-ground vehicles,'' in {\em 2016 American Control Conference
  (ACC)}, pp.~4630--4635, IEEE, 2016.

\bibitem{kapoutsis2017darp}
A.~C. Kapoutsis, S.~A. Chatzichristofis, and E.~B. Kosmatopoulos, ``Darp:
  divide areas algorithm for optimal multi-robot coverage path planning,'' {\em
  J INTELL ROBOT SYST}, vol.~86, no.~3-4, pp.~663--680, 2017.

\bibitem{yakici2017heuristic}
E.~Yak{\i}c{\i}, ``A heuristic approach for solving a rich min-max vehicle
  routing problem with mixed fleet and mixed demand,'' {\em CAIE}, vol.~109,
  pp.~288--294, 2017.

\bibitem{sze2017cumulative}
J.~F. Sze, S.~Salhi, and N.~Wassan, ``The cumulative capacitated vehicle
  routing problem with min-sum and min-max objectives: An effective
  hybridisation of adaptive variable neighbourhood search and large
  neighbourhood search,'' {\em TRANSPORT RES B-METH}, vol.~101, pp.~162--184,
  2017.

\bibitem{narasimha2013ant}
K.~V. Narasimha, E.~Kivelevitch, B.~Sharma, and M.~Kumar, ``An ant colony
  optimization technique for solving min--max multi-depot vehicle routing
  problem,'' {\em Swarm and Evolutionary Computation}, vol.~13, pp.~63--73,
  2013.

\bibitem{wang2016min}
X.~Wang, B.~Golden, E.~Wasil, and R.~Zhang, ``The min--max split delivery
  multi-depot vehicle routing problem with minimum service time requirement,''
  {\em CAIE}, vol.~71, pp.~110--126, 2016.

\bibitem{carlsson2009solving}
J.~Carlsson, D.~Ge, A.~Subramaniam, A.~Wu, and Y.~Ye, ``Solving min-max
  multi-depot vehicle routing problem,'' {\em Lectures on global optimization},
  vol.~55, pp.~31--46, 2009.

\bibitem{zhang2018meta}
S.~Zhang, Y.~Gajpal, and S.~Appadoo, ``A meta-heuristic for capacitated green
  vehicle routing problem,'' {\em Ann. Oper. Res.}, vol.~269, no.~1-2,
  pp.~753--771, 2018.

\bibitem{sundar2017path}
K.~Sundar, S.~Venkatachalam, and S.~G. Manyam, ``Path planning for multiple
  heterogeneous unmanned vehicles with uncertain service times,'' in {\em 2017
  International Conference on Unmanned Aircraft Systems (ICUAS)}, pp.~480--487,
  IEEE, 2017.

\bibitem{venkatachalam2019two}
S.~Venkatachalam, M.~Bansal, J.~M. Smereka, and J.~Lee, ``Two-stage stochastic
  programming approach for path planning problems under travel time and
  availability uncertainties,'' {\em arXiv preprint arXiv:1910.04251}, 2019.

\bibitem{gurobi}
L.~Gurobi~Optimization, ``Gurobi optimizer reference manual,'' 2020.

\bibitem{helsgaun2000effective}
K.~Helsgaun, ``An effective implementation of the lin--kernighan traveling
  salesman heuristic,'' {\em Eur. J. Oper. Res.}, vol.~126, no.~1,
  pp.~106--130, 2000.

\bibitem{khuller2007fill}
S.~Khuller, A.~Malekian, and J.~Mestre, ``To fill or not to fill: the gas
  station problem,'' in {\em European Symposium on Algorithms}, pp.~534--545,
  Springer, 2007.

\bibitem{dijkstra1959note}
E.~W. Dijkstra {\em et~al.}, ``A note on two problems in connexion with
  graphs,'' {\em Numerische mathematik}, vol.~1, no.~1, pp.~269--271, 1959.

\bibitem{hansen2010variable}
P.~Hansen, N.~Mladenovi{\'c}, and J.~A.~M. P{\'e}rez, ``Variable neighbourhood
  search: methods and applications,'' {\em Ann. Oper. Res.}, vol.~175, no.~1,
  pp.~367--407, 2010.

\bibitem{davis1985applying}
L.~Davis, ``Applying adaptive algorithms to epistatic domains.,'' in {\em
  IJCAI}, vol.~85, pp.~162--164, 1985.

\bibitem{gen1996genetic}
M.~Gen and R.~Cheng, {\em Genetic Algorithms and Manufacturing Systems Design}.
\newblock John Wiley \& Sons, Inc., 1996.

\bibitem{schiffer2018adaptive}
M.~Schiffer and G.~Walther, ``An adaptive large neighborhood search for the
  location-routing problem with intra-route facilities,'' {\em Transportation
  Science}, vol.~52, no.~2, pp.~331--352, 2018.

\bibitem{augerat1995computational}
P.~Augerat, J.~M. Belenguer, E.~Benavent, A.~Corber{\'a}n, D.~Naddef, and
  G.~Rinaldi, {\em Computational results with a branch and cut code for the
  capacitated vehicle routing problem}, vol.~34.
\newblock IMAG, 1995.

\bibitem{taguchi1986introduction}
G.~Taguchi, ``Introduction to quality engineering: designing quality into
  products and processes,'' tech. rep., 1986.

\bibitem{Minitab}
``Minitab 19 statistical software (2020). [computer software]. state college,
  pa: Minitab, inc,'' 2020.

\end{thebibliography}
	
\end{document}